\documentclass[11pt]{article}
\usepackage{amsmath, amssymb, amsthm, mathtools}
\usepackage{aliascnt}
\usepackage{colortbl}
\usepackage[dvipsnames]{xcolor}
\usepackage{wrapfig}
\usepackage{pdfpages}
\usepackage{adjustbox}

\usepackage{enumerate}
\usepackage[utf8]{inputenc} 
\usepackage[T1]{fontenc}    
\usepackage{url}            
\usepackage{booktabs}       
\usepackage{amsfonts}       
\usepackage{nicefrac}       
\usepackage{microtype}      
\usepackage[dvipsnames]{xcolor}
\definecolor{mygray}{rgb}{200,200,200}

\usepackage{pdfsync}
\usepackage[OT1]{fontenc}

\usepackage[pagebackref=true, colorlinks, linkcolor=red, anchorcolor=blue, citecolor=blue]{hyperref}
\usepackage{mathrsfs}
\usepackage{fullpage}
\usepackage{float}
\usepackage{subfigure}
\usepackage{setspace}
\usepackage{algorithm}
\usepackage{algorithmic}
\usepackage{balance}
\usepackage{enumitem}

\usepackage{multirow}
\usepackage{wrapfig}

\usepackage{booktabs}  
\usepackage{amsfonts}
\usepackage{nicefrac} 
\usepackage{microtype}
\usepackage{aliascnt}
\usepackage{colortbl}
\usepackage{adjustbox}
\usepackage{verbatim}

\usepackage{verbatim}
\usepackage{eqparbox}
\usepackage{enumitem}
\usepackage{tablefootnote}

\usepackage{eqparbox}

\makeatletter
\newcommand*{\rom}[1]{\expandafter\@slowromancap\romannumeral #1@}

\usepackage{microtype}
\usepackage{mystyle}

\theoremstyle{plain}
\newtheorem{theorem}{Theorem}[section]

\newtheorem{lemma}[theorem]{Lemma}

\theoremstyle{definition}
\newtheorem{definition}[theorem]{Definition}

\theoremstyle{remark}

\usepackage{amsmath,amsfonts,bm}

\def\eqref#1{equation~\ref{#1}}

\def\1{\bm{1}}

\DeclareMathAlphabet{\mathsfit}{\encodingdefault}{\sfdefault}{m}{sl}
\SetMathAlphabet{\mathsfit}{bold}{\encodingdefault}{\sfdefault}{bx}{n}

\title{\huge Personalized Federated Training of Diffusion Models with Privacy Guarantees}
\author{
  Kumar Kshitij Patel\\
  TTIC\\
  \texttt{kkpatel@ttic.edu} \\
    \and Weitong Zhang\\
  UNC Chapel Hill\\
  \texttt{weitongz@unc.edu}
  \and Lingxiao Wang\\
  NJIT\\
  \texttt{lw324@njit.edu}
}
\date{}

\begin{document}
\maketitle

\begin{abstract}
The scarcity of accessible, compliant, and ethically sourced data presents a considerable challenge to the adoption of artificial intelligence (AI) in sensitive fields like healthcare, finance, and biomedical research. Furthermore, access to unrestricted public datasets is increasingly constrained due to rising concerns over privacy, copyright, and competition. Synthetic data has emerged as a promising alternative, and diffusion models---a cutting-edge generative AI technology---provide an effective solution for generating high-quality and diverse synthetic data. In this paper, we introduce a novel federated learning framework for training diffusion models on decentralized private datasets. Our framework leverages personalization and the inherent noise in the forward diffusion process to produce high-quality samples while ensuring robust differential privacy guarantees. Our experiments show that our framework outperforms non-collaborative training methods, particularly in settings with high data heterogeneity, and effectively reduces biases and imbalances in synthetic data, resulting in fairer downstream models.
\end{abstract}

\section{Introduction}\label{sec:intro}
\renewcommand{\thefootnote}{\fnsymbol{footnote}}
\footnotetext{Preliminary results.}
The proliferation of foundation models has the potential to revolutionize artificial intelligence (AI), driving advancements in applications such as healthcare diagnostics, financial systems, scientific discovery, and large-scale automation~\citep{bommasani2021opportunities,kapoor2024societal}. However, their success depends heavily on access to large-scale, high-quality datasets~\citep{kaplan2020scaling}---a dependency increasingly being challenged by two major issues. First, the \textbf{decentralization of data} and tightening legal and regulatory restrictions, including stricter copyright protections, equitable compensation demands, and privacy laws, are limiting data accessibility~\citep{gdpr,cpra,grynbaum2023times,wang2023generative}. Consequently, the era of openly available and loosely regulated public data is seemingly coming to an end~\citep{villalobosposition}. Second, domains such as healthcare and drug discovery face an inherent \textbf{scarcity of data}, as rare medical conditions and the high cost of testing unique chemical compounds on humans make it infeasible to collect large, diverse datasets~\citep{ching2018opportunities,rieke2020future,clark2021machinelearning,prakash2023chemicalspace}. These challenges amplify issues of bias and fairness, as models trained on limited or imbalanced datasets often fail to generalize to underrepresented populations or critical edge cases. 
Together, they render traditional data collection and centralized training pipelines impractical, underscoring the need for innovative approaches to data generation and collaborative model training.

Federated learning provides a promising framework for addressing the challenges of decentralized, siloed datasets by enabling multiple institutions or clients to collaboratively train a shared model without exchanging raw data~\citep{mcmahan2016communication,mcmahan2016federated,kairouz2021practical}. This approach overcomes privacy concerns and legal barriers, making it particularly valuable in regulated domains. However, federated learning alone cannot solve the problem of data scarcity. Diffusion models~\citep{ho2020denoising,sohl2015deep}, a cutting-edge generative AI technology, offer a complementary solution by generating diverse and realistic data samples from complex distributions. Their ability to model intricate variability within datasets makes them a powerful tool for addressing data scarcity, especially in fields where obtaining large, representative datasets is inherently difficult.

Thus, combining federated learning with generative modeling represents a natural step forward, enabling collaborative training of generative models across decentralized datasets while addressing both data accessibility and scarcity challenges.

Current approaches for federated training of diffusion models~\citep{vora2024feddm,de2024training} focus on training a single global model using only client gradients, which leverages decentralized datasets and avoids the exchange of raw data. However, this protocol does not inherently guarantee privacy, as diffusion models are prone to memorizing their training data~\citep{chou2023backdoor,carlini2023extracting,pang2023white}, and standard privacy-preserving techniques, such as differentially private (DP) training of diffusion models, have yet to be effectively applied, increasing the risk of exposing sensitive information.
Furthermore, a single global model trained on the mixture of all client distributions enables any client (or external user) to generate synthetic data that reflects patterns specific to other clients, which is not ideal for scenarios with competitive pressures. Thus, addressing these challenges necessitates both mitigating \textbf{memorization risks} and introducing mechanisms for \textbf{client-specific control} over synthetic data generation. Without such guarantees, the applicability of federated diffusion models in sensitive domains such as healthcare remains limited.

\begin{figure}[!h]
  \begin{center}
    \includegraphics[width=0.6\textwidth]{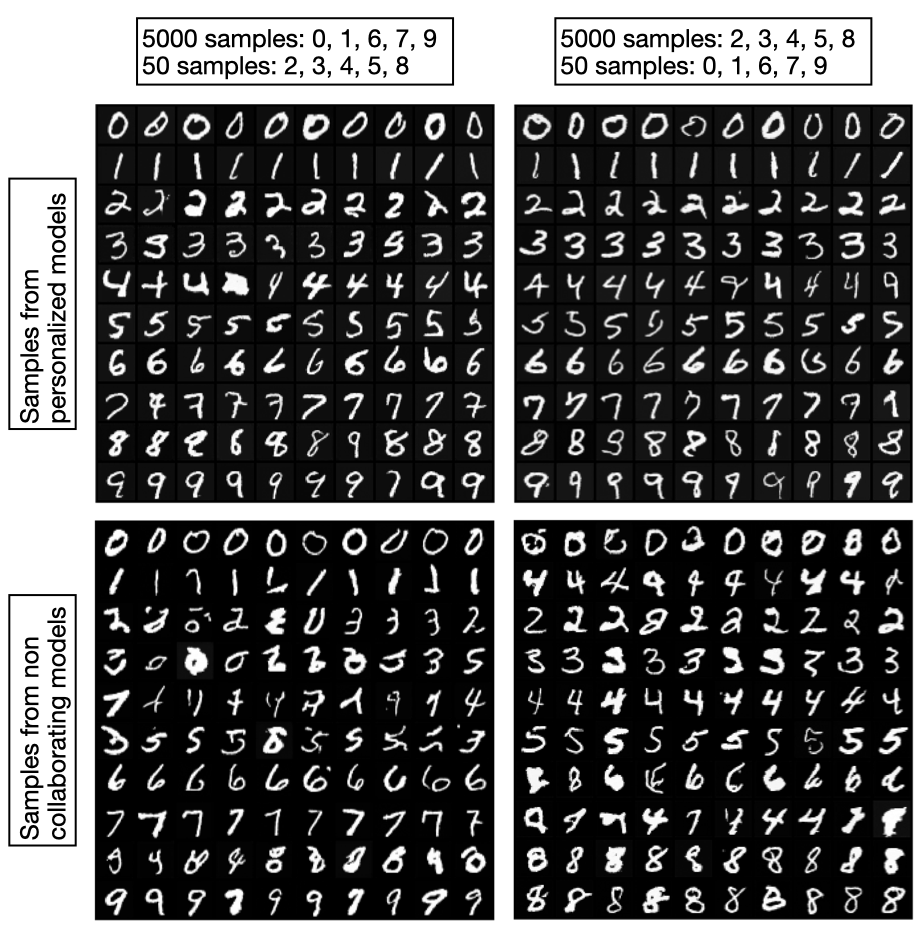}
  \end{center}
  \label{fig:MNIST}
  \caption{\textbf{The effect of data heterogeneity on generation.} We divide the MNIST dataset into two clusters, each containing five classes. We then construct two distinct datasets, each comprising 5,000 samples from one cluster and 50 from the other, with the majority-minority clusters reversed between datasets. Our goal is to perform conditional image generation for all classes. We explore two training approaches: (1) training a single diffusion model independently on each dataset (second row) and (2) using our personalized framework, which trains two models per dataset---a shared global model and a personalized local model (first row). We observe that models trained without collaboration perform significantly worse on the minority classes and often generate majority-class images even when prompted with a minority-class label. For example, the model confuses a 4 with a 7 in the bottom-left figure. In contrast, the personalized framework produces higher-quality images for the minority classes, demonstrating that the global model captures and transfers shared features across datasets. We discuss privacy guarantees in Section \ref{sec:result}.}
\end{figure}

\paragraph{Our Contribution.} We propose a novel methodology for collaboratively training diffusion models across decentralized, private datasets that addresses the above limitations. Our approach uses personalization by splitting the reverse diffusion process (a.k.a. de-noising) into client-specific and global components. The client diffusion model learns to map noisy client images to clean ones, while the global diffusion model maps standard Gaussian noise to a mixture of these noisy client images. As a result, the global model only ever processes noisy images from the clients, ensuring that sensitive information remains obfuscated and naturally mitigating memorization risks. This separation enables clients to maintain precise control over the synthetic data generation process, tailoring it to their needs while ensuring the global model reveals no sensitive patterns. Furthermore, the global model captures higher-level generalizable features across the datasets, which helps mitigate bias in the synthetic data due to imbalance in the clients' datasets.  

We formalize our privacy guarantee as a function of the diffusion process parameters and the degree of personalization in Theorem \ref{thm:privacy}. Our guarantee offers a differential privacy (DP) guarantee for each client's data. Through experiments on CIFAR-10 and MNIST (e.g., Figure \ref{fig:MNIST}), we show that our method generates images comparable to centralized training, i.e., using the combined dataset across all clients (the non-private extreme). 
Additionally, we show that our method significantly outperforms models trained on individual client datasets (the private extreme) by varying dataset sizes and data heterogeneity across clients.

\section{Related Work}\label{sec:related}
\paragraph{Image Generation and Diffusion Models.} 
Image generation is a fundamental problem in machine learning, with applications from art and animation to medical imaging and molecular design. The task involves synthesizing realistic images from underlying data distributions, and as such, a variety of techniques have been developed over time. Explicit density estimation methods, such as autoregressive models like \textbf{PixelRNN} and \textbf{PixelCNN}~\citep{van2016pixel, van2016conditional}, model the joint distribution of pixel values but often suffer from slow sampling and limited scalability. \textbf{Variational Autoencoders (VAEs)}~\citep{kingma2013auto,van2017neural} introduced latent-variable modeling, balancing reconstruction, and regularity but often produced blurry samples due to the smoothing effect of their approximations. \textbf{Generative Adversarial Networks (GANs)}~\citep{goodfellow2014generative} tackle this issue with adversarial training, achieving state-of-the-art sample quality, but at the cost of training instability.

Recent advancements in score-based and diffusion probabilistic models have emerged as robust alternatives. Diffusion models, particularly \textbf{Denoising Diffusion Probabilistic Models (DDPM)}~\citep{ho2020denoising}, define a generative process as a sequence of denoising steps that reverse a forward process where noise is incrementally added to data. Unlike GANs, DDPMs are stable to train and provide a principled probabilistic framework, allowing them to surpass GANs in image generation quality~\citep{dhariwal2021diffusion}. Extensions of DDPMs have further enhanced scalability and sample efficiency, enabling their application to high-resolution synthesis, inpainting, and conditional generation~\citep{song2020improved, nichol2021improved}.

In this work, we focus on DDPMs due to their simplicity and foundational role within the broader class of diffusion models. While our methodology is demonstrated with DDPMs, the core ideas extend to more advanced diffusion-based models. For an introduction to the algorithmic foundation of DDPMs, we refer readers to Section~\ref{sec:pre}, and for a broader overview of diffusion models, we recommend existing surveys~\citep{karras2022elucidating, nakkiran2024step}.

\paragraph{Federated Learning of Diffusion Models.} Federated learning (FL) was introduced to enable collaborative training on decentralized datasets while preserving privacy by avoiding the exchange of raw data~\citep{mcmahan2016communication,mcmahan2016federated,kairouz2021practical}, primarily in response to data regulations such as GDPR~\citep{gdpr}. FL has been widely adopted in domains like healthcare, research, finance, and mobile technologies~\citep{bergen2012genome,dayan2021federated,li2019privacy,roth2020federated,wang2023scientific,chen2020fl,kaissis2021end,shiffman_zarate_deshpande_yeluri_peiravi_2021,paulik2021federated}. The standard FL protocol involves exchanging gradients between clients and the server, which provides a degree of privacy but cannot prevent gradient-based~\citep{wang2023reconstructing} or data regeneration attacks---a known vulnerability of diffusion models~\citep{carlini2023extracting}.

While early FL research primarily focused on predictive models, recent efforts have explored its application to generative models, particularly diffusion models. Existing works in this space can be broadly categorized into those that apply standard FL techniques to train a single global diffusion model and those that leverage diffusion models as an auxiliary tool for specific FL applications. Only the first category is relevant to our work, but we also survey the latter for completeness.

Most existing works in the first category adopt standard FL methods, such as Federated Averaging (FedAvg), to train a shared diffusion model across clients. \citet{de2024training} proposed a framework that extends FedAvg to train DDPMs, demonstrating that a global model can achieve image quality close to centralized training while reducing communication costs. \citet{vora2024feddm} introduced FedDM, which enhances communication efficiency and robustness to heterogeneous data distributions by incorporating quantization and proximal regularization techniques. \citet{lai2024demand} further improved quantization strategies for cross-device FL applications. However, these methods focus solely on training a single global model and do not address privacy risks or client-specific control over synthetic data generation. \citet{sattarov2024differentially,sattarov2024fedtabdiff} used DP-SGD in the FedAvg updates for training their diffusion models on tabular data, providing some level of privacy, but still no client-specific control. In the non-federated setting (i.e., with a single dataset), \citet{wang2024dp} developed a training protocol that selectively uses DP-SGD only to train the final stages of de-noising, thus claiming to improve the privacy-utility trade-off of DP-trained diffusion models. However, since their method observes the actual noise used in forward diffusion, their approach does not ensure differential privacy guarantees.

\paragraph{Personalization and Differential Privacy.} Personalization is a widely used strategy in FL to address data heterogeneity and provide client-specific control~\citep{mishchenko2023partially}. Most personalized FL frameworks partition the model into shared components, trained collaboratively, and local components tailored to individual clients. Personalization can also enhance the privacy-utility trade-off, as an appropriately designed shared model can focus on generalizable patterns while leaving fine-grained, client-specific features to the local model~\citep{bietti2022personalization}. \citet{jothiraj2023phoenix} applied this approach to federated diffusion models by splitting the UNet architecture into shared and local components. \citet{chen2024fedbip} proposed to use a more involved personalization to utilize diffusion models for one-shot fine-tuning on clients in the federated setting. Unfortunately, neither of these works provides formal privacy guarantees beyond those inherent to standard federated training. 
In this paper, we leverage the inherent noise in the forward diffusion process, proposing a novel federated training framework that balances privacy, personalization, and sample quality.

\section{Preliminaries}\label{sec:pre}
\subsection{Denoising Diffusion Probabilistic Models (DDPM)}
 Diffusion models aim to learn a distribution $p_{\theta}$ that approximates a target distribution $q_0$ of interest. Their learning process consists of a forward diffusion process and a reverse process. In the diffusion process, a sample $x_0\sim q_0$
is sequentially corrupted by random Gaussian noise, and with enough time steps, the
data distribution is transformed into pure noise. In
the reverse process, a neural network is
trained to denoise this pure noise, i.e., remove the added noise sequentially to recover a new data distribution. The hope is that this recovered distribution is ``close" to $q_0$ and allows generating new samples from $q_0$ not seen during training. 

Denoising Diffusion Probabilistic Models (DDPM) \citep{ho2020denoising} are one of the most popular diffusion models. Their diffusion process $q(x_t|x_{t-1})$, which produces corrupted latents $\{x_t\}_{t\in[T]}$, is defined as a Markov chain that sequentially adds Gaussian noise to the data according to the noise schedule $\{\beta_t\}_{t\in[T]}$:
\begin{align}\label{eq:ddpm_forward}
    q(x_t|x_{t-1}):=\cN(x_t;\sqrt{1-\beta_t}x_{t-1},\beta_t\Ib).
\end{align}
Given large enough $T$ and appropriate noise schedule, the latent $x_T$ nearly follows a standard Gaussian distribution. Furthermore, the diffusion process in \eqref{eq:ddpm_forward} implies 
\begin{align}\label{eq:diffuse_1step}
   q(x_t|x_0)=\cN(x_t;\sqrt{\bar \alpha_t}x_{0},(1-\bar\alpha_t)\Ib), 
\end{align}
 where $\bar \alpha_t=\prod_{s=1}^t(1-\beta_s)$. Thus, we can directly sample an arbitrary latent $x_t$ given $x_0$.

The reverse process $q(x_{t-1}|x_t)$ can be approximated by:
\begin{align}\label{eq:ddpm_reverse}
    p_\theta(x_{t-1}|x_t):=\cN(x_t;\mu_{\theta}(x_t,t),\Sigma_{\theta}(x_t,t)),
\end{align}
where $\mu_{\theta}(x_t,t)$ is a trainable neural network with model parameter $\theta$ and $\Sigma_{\theta}(x_t,t)$ can be set to $\sigma^2_t\Ib$ \citep{ho2020denoising} (see \cite{nakkiran2024step} for an intuition about why this is true).

To learn $p_\theta$, we can minimize the variational upper bound $\EE_q[-\log p_{\theta}(x_{0:T})/q(x_{1:T}|x_0)]$ which is equivalent to minimizing the following sum of KL divergences \citep{sohl2015deep, ho2020denoising}
\begin{align}\label{eq:KL_object}
    \EE_q\Big[\textstyle{\sum_{t>1}}D_{\text{KL}}\big(q(x_{t-1}|x_t,x_0)||p_{\theta}(x_{t-1}|x_t)\big)\Big],
\end{align}
where $q(x_{t-1}|x_t,x_0)=\cN(x_{t-1};\tilde \mu_t(x_t,x_0),\tilde\beta_t\Ib)$ and  $\tilde \mu_t(x_t,x_0),\ \tilde \beta_t$ can be parameterized \citep{ho2020denoising} as:
\begin{align*}
    \tilde \mu_t(x_t,x_0)=\frac{1}{\sqrt{1-\beta_t}}\bigg(x_t-\frac{\beta_t}{\sqrt{1-\bar\alpha_t}}z\bigg),~\tilde \beta_t=\frac{1-\bar\alpha_{t-1}}{1-\bar\alpha_{t}}\beta_t,
\end{align*}
where $z\sim\cN(0,\Ib)$.

Hence, \citet{ho2020denoising} proposed to represent $\mu_{\theta}(x_t,t)$ in $p_{\theta}$ using a noise prediction network, i.e., denoiser, $z_\theta$:
\begin{align*}
   \mu_{\theta}(x_t,t)= \frac{1}{\sqrt{1-\beta_t}}\bigg(x_t-\frac{\beta_t}{\sqrt{1-\bar\alpha_t}}z_\theta(x_t,t)\bigg),
\end{align*}
and the objective in \eqref{eq:KL_object} can be simplified to the following loss
$$ \EE_{t,x_0,z}\big[\|z-z_\theta(x_t,t)\|_2^2\big],$$
where $t$ is uniformly sampled from $1$ to $T$. The detailed training procedure is summarized in Algorithm \ref{alg:DDPM}. After obtaining the denoiser $z_\theta$, we can generate a sample by drawing $x_T\sim \cN(0,\Ib)$ and using the form of $p_{\theta}$. The detailed sampling procedure is illustrated in Algorithm \ref{alg:sampling_ddpm}.
\begin{algorithm}[t]
\caption{Training Procedure for DDPM  (T-DDPM)}\label{alg:DDPM}
\begin{algorithmic}[1]
\INPUT Training dataset $D$, model parameters $\theta$, time step $T$, noise scheduling parameters $\{\beta_t\}_{t=1}^{T}$
\REPEAT
\STATE Sample $x_0$ from $D$
\STATE Sample $t \sim \text{Uniform}(\{1, \dots, T\})$
\STATE Sample random noise $z \sim \cN(0,\Ib)$
\STATE Set $\bar \alpha_t=\prod_{s=1}^t(1-\beta_s)$
\STATE Take gradient descent step on\\
\qquad$\nabla_{\theta}\big\|z-z_\theta\big(\sqrt{\bar\alpha_{t}}x_0+\sqrt{1-\bar\alpha_{t}}z,t\big)\big\|_2^2$
\UNTIL converged
\OUTPUT  denoiser $z_{\theta}$
\end{algorithmic}
\end{algorithm}

\begin{algorithm}[t]
\caption{Sampling Procedure for DDPM (S-DDPM)}\label{alg:sampling_ddpm}
\begin{algorithmic}[1]
\INPUT Denoiser $z_\theta$, time step $T$, noise scheduling parameters $\{\beta_t,\sigma_t\}_{t=1}^{T}$
\STATE Set $\bar \alpha_t=\prod_{s=1}^t(1-\beta_s)$
\STATE Sample $x_T \sim \cN(0,\Ib)$ 
\FOR{$t=T,\ldots,1$}
\STATE Sample $z \sim \cN(0,\Ib)$ if $t>1$ else $z=0$
\STATE $x_{t-1}=\frac{1}{\sqrt{1-\beta_t}}\Big(x_t-\frac{\beta_t}{\sqrt{1-\bar \alpha_t}}z_{\theta}(x_t,t)\Big)+\sigma_t z$
\ENDFOR
\OUTPUT  $x_0$
\end{algorithmic}
\end{algorithm}

\subsection{Differential Privacy}
We will show that our proposed method can ensure differential privacy guarantees. We introduce two different notions of differential privacy (DP), central differential privacy \citep{dwork2006calibrating} and local differential privacy \citep{kasiviswanathan2011can}.

\begin{definition}[$(\epsilon,\delta)$-DP]\label{def:DP}
A randomized mechanism $\cA$ satisfies $(\epsilon,\delta)$-differential privacy if for adjacent datasets $D,D'$ differing by one element, and any output subset $O$, it holds that 
$$\PP[\cA(D)\in O]\leq e^\epsilon\cdot \PP[\cA(D')\in O]+\delta.$$
\end{definition}
The idea of central DP is to ensure that any output should be about as likely (controlled by $\epsilon$)
regardless of whether an individual's data is included in the dataset or not. $\epsilon$ is known as the privacy budget, and a
smaller $\epsilon$ enforces a stronger privacy guarantee. In this paper, we focus on the following local differential privacy \citep{kasiviswanathan2011can}. 
\begin{definition}[$(\epsilon,\delta)$-LDP]\label{def:LDP}
A randomized mechanism $\cA$ satisfies $(\epsilon,\delta)$-local differential privacy if for any two inputs $x,x'$, and any output subset $O$, it holds that $$\PP[\cA(x)\in O]\leq e^\epsilon\cdot \PP[\cA(x')\in O]+\delta.$$
\end{definition}
The idea of local DP is to ensure that any output should be
about as likely regardless of an individual's data. Therefore, local DP often provides a much stronger privacy protection for the individual's data compared to central DP.

\section{Proposed Method}\label{sec:method}
Our proposed training framework, i.e., PFDM in Algorithm \ref{alg:train_feddiff}, consists of two stages. At the first stage, each client $m\in[M]$ will train a personalized denoiser $z_{\theta_m}$ based on its training dataset $D_m$ (line 2 in Algorithm \ref{alg:train_feddiff}), and this model will not be shared with other clients or the server. $z_{\theta_m}$ is trained using the DDPM training procedure (Algorithm \ref{alg:DDPM}) with $t_0$ time steps.  
After obtaining $z_{\theta_m}$, a noisy dataset $\tilde D_m$ is created with each data point generated through the diffusion process (lines 3-8 in Algorithm \ref{alg:train_feddiff}). The noisy dataset will then be sent to the server.
\begin{algorithm}[ht]
\caption{\textbf{P}ersonalized \textbf{F}ederated Training of \textbf{D}iffusion \textbf{M}odel (PFDM)}\label{alg:train_feddiff}
\begin{algorithmic}[1]
\INPUT Training dataset $D_m$ for each client $m \in [M]$, shared model parameter $w$, client-specific model parameters $\{\theta_m\}_{m=1}^M$, local time step $t_0$, global time step $T$, noise schedule $\{\beta_t\}_{t=1}^{T}$
\FORC{$m \in [M]$} 
    \STATE Train a personalized secret denoiser $z_{\theta_m}=\text{T-DDPM}(D_m, t_0, \{\beta_t\}_{t=1}^{T},\theta_m)$
    \STATE Sample $N$ data from $D_m$ indexed by $\cB_{m}$
    \STATE Set $\bar \alpha_t=\prod_{s=1}^t(1-\beta_s)$
    \FOR{$i \in \cB_{m}$} 
        \STATE Obtain noisy data $\tilde{x}_0^{i,m}=\sqrt{\bar\alpha_{t_0}}x_0^{i,m}+\sqrt{1-\bar\alpha_{t_0}}z$, where $z \sim \cN(0,\Ib)$
    \ENDFOR
    \STATE \textbf{Send (communicate)} $\tilde D_m=\{\tilde{x}_0^{i,m}\}_{i\in \cB_{m}}$ to server
\ENDFORC
\FORS{}
    \STATE Obtain $\tilde D$ by combining $\{\tilde D_m\}_{m\in[M]}$ 
    \STATE  Train a shared global denoiser 
    $z_{w}=\text{T-DDPM}(\tilde D, T, \{\beta_t\}_{t=1}^{T}, w)$
\ENDFORS
\OUTPUT  shared global denoiser  $z_{w}$
\end{algorithmic}
\end{algorithm}
At the second stage, the server will train a global denoiser based on the collected noisy datasets $\{\tilde D_m\}_{m\in[M]}$ from all clients (lines 11-14 in Algorithm \ref{alg:train_feddiff}). Note that the diffusion process (line 7 in Algorithm 
 \ref{alg:train_feddiff}) can ensure certain level of differential privacy guarantee for each data point, and thus the trained global denoiser $z_w$ is also differentially private. Therefore, $z_w$ can be published and shared among clients. We formally provide such guarantees in the next subsection.

After obtaining the differentially private shared global denoiser $z_w$ and personalized secret denoisers $\{z_{\theta_m}\}_{m\in[M]}$, we can generate samples for each client. The detailed sampling procedure is illustrated in Algorithm \ref{alg:sampling_feddiff}. More specifically, for client $m$, it will first receive the sample $\tilde x_0$ generated using the global denoiser $z_w$. Then the client will use its personalized secret denoiser $z_{\theta_m}$ to denoise the receive sample $\tilde x_0$ for another $t_0$ steps to get the synthetic sample that the client wants.
\begin{algorithm}[ht]
\caption{Sampling Procedure for Personalized Federated Diffusion Model}\label{alg:sampling_feddiff}
\begin{algorithmic}[1]
\INPUT Shared global denoiser $z_w$, personalized denoiser $z_{\theta_m}$, global time step $T$, local time step $t_0$ noise scheduling parameters $\{\beta_t,\sigma_t\}_{t=1}^{T}$
\STATE $\tilde x_0=\text{S-DDPM}(z_w,T,\{\beta_t,\sigma_t\}_{t=1}^{T})$
\STATE Set $\bar \alpha_t=\prod_{s=1}^t(1-\beta_s)$
\STATE $x_{t_0}=\tilde x_0$
\FOR{$t=t_0,\ldots,1$}
\STATE $x_{t-1}=\frac{1}{\alpha_t}\Big(x_t-\frac{1-\alpha_t}{\sqrt{1-\bar \alpha_t}}z_{\theta_m}(x_t,t)\Big)+\sigma_t z$, where $z \sim \cN(0,\Ib)$ if $t>1$ else $z=0$
\ENDFOR
\OUTPUT  $x_0$
\end{algorithmic}
\end{algorithm}

We expect each client to generate high-quality desired samples and protect the privacy of their training data using the denoisers outputted by our PFDM framework. This is because the reverse process $p_w$ (line 1 in Algorithm \ref{alg:sampling_feddiff}) with the shared global denoiser $z_w$ aims to approximate the mixture of $M$ diffused data distributions $\{q_m(x_{t_0})\}_{m\in[M]}$, where $q_m(x_{t_0})=\int q_m(x_0)q(x_{t_0}|x_0)dx_0$, $q_m(x_0)$ is the input data distribution for client $m$, and $q(x_{t_0}|x_0)$ is defined in \eqref{eq:diffuse_1step}. Take a natural image as one example. During the diffusion process, the image's fine-grained details (e.g., textural details) are perturbed faster than the macro structures such as the background \citep{rissanen2022generative}. As a result, the diffused data distributions $\{q_m(x_{t_0})\}_{m\in[M]}$ capture the overall structure distributions, which are often similar to each other. As a result, the reverse process $p_w$, which is trained based on $M$ training datasets with each $\tilde D_m$ drawn from $q_m(x_{t_0})$, can approximate a shared distribution that is useful for each client. Since the denoiser $z_w$ used in $p_w$ is trained on the noisy data, we get privacy protections for the clean training datasets on the devices.

On the other hand, the personalized secret denoiser $z_{\theta_m}$ is trained to denoise the image from $x_{t_0}$ to generate fine-grained details (desired by the specific client $m$) in the image. Since the fine-grained details often contain sensitive information, we propose to train $z_{\theta_m}$ using client $m$'s data and keep it secret. As a result, each client can generate a high-quality desired sample where the overall structure in the generated image can benefit from the collaborative training of the global model, and the fine-grained details can be desired by using the personalized local model. Note that intuitively, the less noisy the dataset for training the global model, the easier it is to train the local model on each device, but the weaker our privacy guarantee is. Thus, there is a trade-off between the complexity of client denoising and the privacy guarantee, which we formalize in the next section.

\section{Main Results}\label{sec:result}
We show that our proposed PFDM algorithm can ensure local differential privacy.
\begin{theorem}[Privacy Guarantee of \textsc{PFDM}]\label{thm:privacy}
Given the training dataset $D=\{D_m\}_{m\in[M]}$, if we choose the local time step to be $t_0$, global time step to be $T$, noise scheduling parameters to be $\{\beta_t\}_{t=1}^{T}$, then the output of Algorithm \ref{alg:train_feddiff} is $\Big(\frac{2\bar \alpha_{t_0}C^2}{1-\bar \alpha_t}+C\sqrt{\frac{8\bar \alpha_{t_0}\log(1/\delta)}{1-\bar \alpha_{t_0}}},\delta\Big)$-LDP for all $x \in D$, where $\bar \alpha_t=\prod_{s=1}^t(1-\beta_s)$ and $C=\|x\|_2$. 
\end{theorem}
Consider $D$ to be the MNIST dataset. Suppose for a given training data $x$, we have  $\|x\|_2\leq 10$. If we choose $T$ to be 1000 steps, $\beta_t$ to be a linear schedule as in the standard DDPM training, then Theorem \ref{thm:privacy} implies that when $t_0=400$, our proposed PFDM algorithm is $(\epsilon,\delta)$-LDP with $\epsilon=95$ and $\delta=10^{-5}$. Although $\epsilon=95$ seems to be too large to provide any meaningful privacy guarantee, the protection we provide here is for the whole image $x$ with $784$ pixels. In many practical scenarios, we aim to protect specific pixels. In that case our method can provide $\Big(\frac{2\bar \alpha_{t_0}c^2}{(1-\bar \alpha_t)}+c\sqrt{\frac{8\bar \alpha_{t_0}\log(1/\delta)}{(1-\bar \alpha_{t_0})}},\delta\Big)$-DP for each pixel in $x$, where $c$ is the maximum value over all pixels in $x$. Suppose $c=1$, we have $\epsilon=5.2$ for each pixel. Furthermore, if we want to protect specific $k$ pixels, according to group privacy, our method can provide $\Big(k\epsilon_1+k\epsilon_2\sqrt{5+k(\epsilon_1+\epsilon_2)},10^{-5}\Big)$-DP for $k$ pixels in $x$, where $\epsilon_1=\frac{2\bar \alpha_{t_0}c^2}{(1-\bar \alpha_t)}$ and $\epsilon_2=\sqrt{\frac{8c^2\bar \alpha_{t_0}}{(1-\bar \alpha_{t_0})}}$. If we choose $k=10$, we have $\epsilon=72$ DP privacy guarantees for those pixels. In practice, if we want to use membership inference attacks to determine the membership of a given image or some of its pixels, two digits $\epsilon$ DP guarantee can effectively defend against many strong membership inference attacks \citep{jayaraman2020revisiting,lowy2024does}.

Note $t_0$ serves as a knob to control the privacy-utility trade-offs. The larger the $t_0$ we have, the smaller the $\bar \alpha_{t_0}$ we get, the stronger the privacy guarantee we can provide, and the worse our generated sample quality. The shared global model we used to generate samples in Figure \ref{fig:MNIST} can achieve $\epsilon=16.6$ DP for each pixel, and the samples generated using the global denoiser are shown in Figure \ref{fig:global_mnist}.
\begin{figure}[ht]
    \centering
\includegraphics[width=1.0\linewidth]{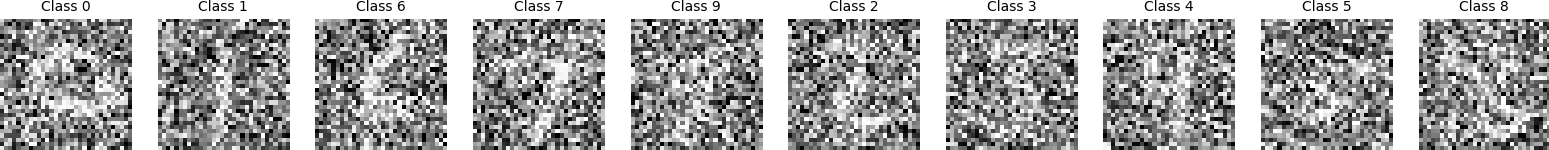}
\vspace{-0.2in}
    \caption{Illustration of generated samples using shared global denoiser in Figure \ref{fig:MNIST}.}    
    \label{fig:global_mnist}
\end{figure}

\section{Experiments}\label{sec:exp}
This section presents preliminary results on the CIFAR-10 dataset to evaluate our proposed method. 

\paragraph{Data Preparation.} We divide the CIFAR-10 dataset into two clusters: one with 4 classes (airplane, car, ship truck) and the other one with 6 classes (bird, cat, deer, dog, frog, horse). We then construct two distinct datasets, each comprising 5,000 samples from one cluster and 50 from the other, with the majority-minority clusters
reversed between the datasets.  

\paragraph{Conditional Generation.} Our goal is to perform conditional image generation for all classes. In our setting, we aim to protect the image (or its pixels) rather than its label information. Therefore, we treat the image label as the publicly available information. We included such label information during the training, which allowed us to obtain a conditional diffusion model.

\paragraph{Baselines.} We compare our method with the non-collaborative DMs, where we train a single diffusion model independently on each
dataset. We also present the results of the non-private DMs, where we train one diffusion model by combining two distinct datasets. We use the same linear noise schedule for all methods. For our method, we choose $t_0=100, T=1,000$. This leads to $\epsilon=45$ DP guarantees for each pixel of a given image in the CIFAR-10 training dataset according to Theorem \ref{thm:privacy}.

 \begin{table}[!t]
\caption{FID scores for different methods. For each method, we report the results for the model trained with cluster 1 (i.e., airplane, car, ship, truck) as majority class.}
\label{table:fid}
\begin{center}
\begin{small}
\begin{sc}
\begin{tabular}{clc}
\toprule
Generated Class & Methods & FID  \\
\midrule
\multirow{3}{*}{Airplane, car, Ship, Truck} & Non-private DMs  & 15.61 \\
                         & Non-collaborative DMs    & 16.60   \\& \textbf{Ours }    & \textbf{14.75}  \\
\midrule
\multirow{3}{*}{Deer, Dog, Frog, Horse} & Non-private DMs  & \textbf{15.60}  \\
                         & Non-collaborative DMs  & 23.67   \\
& \textbf{Ours }    & 17.31  \\
\bottomrule
\end{tabular}
\end{sc}
\end{small}
\end{center}
\end{table}

\begin{table}[t]
\caption{Test accuracies for different methods. * Our method is significantly better than the non-collaborative DMs according to the significance test.}
\label{table:test_acc}
\begin{center}
\begin{small}
\begin{sc}
\begin{tabular}{clc}
\toprule
 Methods & Test Accuracy  \\
\midrule
Non-private DMs  & 56.80  $\pm$ 1.1\\
Non-collaborative DMs  & 54.96  $\pm$ 2.5 \\
\textbf{Ours}  & \textbf{57.70}  $\pm$ \textbf{2.1} (*)\\
\midrule
Original CIFAR-10  & 71.56  $\pm$ 0.3\\
\bottomrule
\end{tabular}
\end{sc}
\end{small}
\end{center}
\end{table}

\begin{figure}[!t]
    \centering
    \subfigure[Ours]{%
        \includegraphics[width=0.5\textwidth]{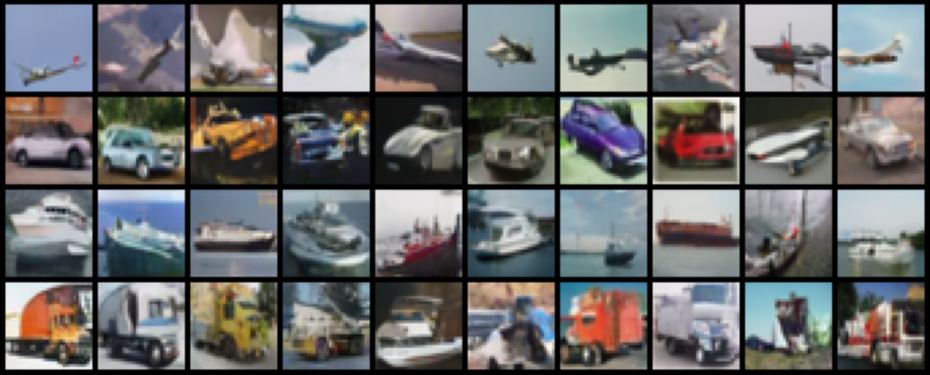}
    }%
    \subfigure[Non-collaborative DMs]{%
        \includegraphics[width=0.5\linewidth]{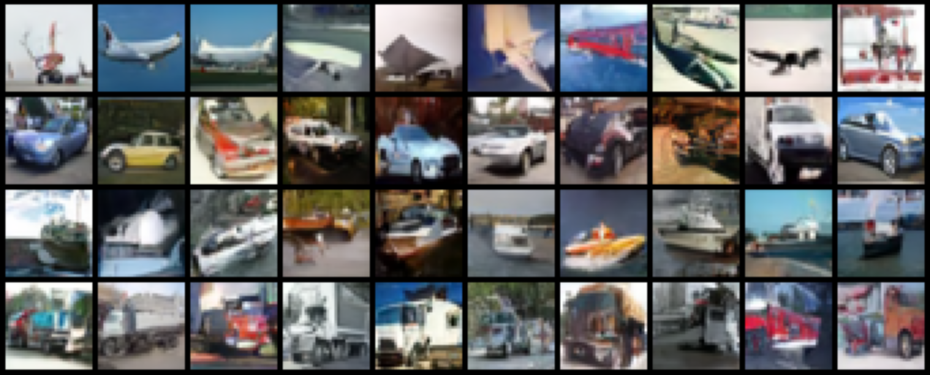}
    }
    \subfigure[Ours]{%
        \includegraphics[width=0.5\linewidth]{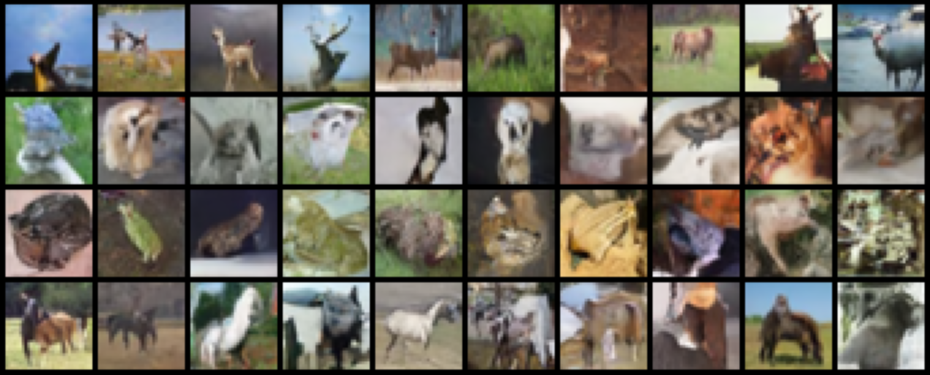}
    }%
    \subfigure[Non-collaborative DMs]{%
        \includegraphics[width=0.5\linewidth]{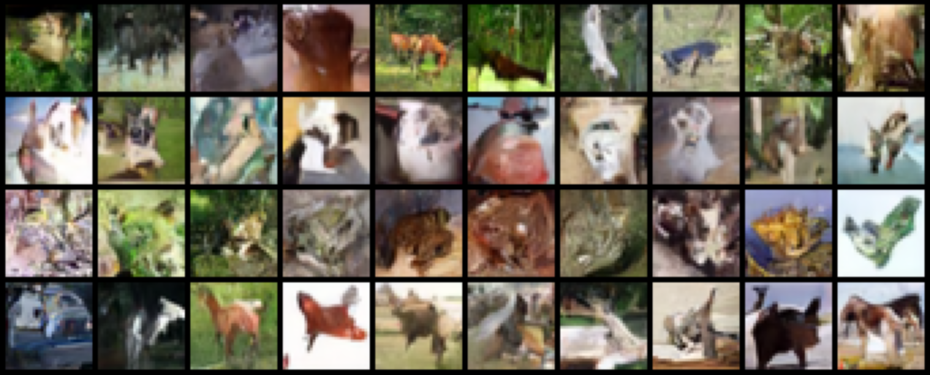}
    }
    \vspace{-0.1in}
    \caption{CIFAR-10 samples generated by different methods. (a), (b) correspond to samples generated for classes airplane, car, ship, truck. (c), (d) correspond to samples generated for classes deer, dog, frog, horse. We report the results for the
model trained with cluster 1 as majority class. }
    \label{fig:1}
\end{figure}

\begin{figure}[!t]
    \centering
\includegraphics[width=1.0\linewidth]{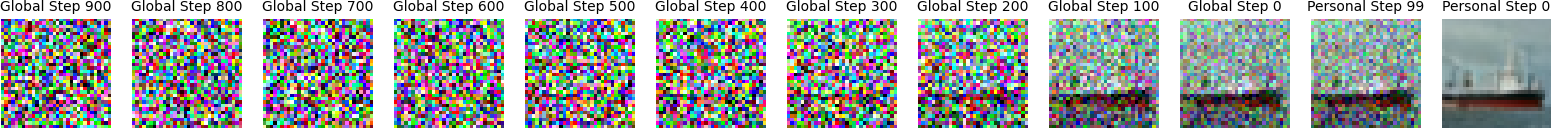}
\vspace{-0.2in}
    \caption{Illustration of our sampling procedure. The global step 0 correspond to the output using the shared global denoiser. The personal steps correspond to the outputs of using personalized denoiser.}
    \label{fig:3}
\end{figure}
\paragraph{Evaluation Metrics.} To evaluate the performance of different methods, we use two metrics. The first one is the Fréchet inception distance (FID), which is widely used to evaluate the image quality in the literature \citep{heusel2017gans,ho2020denoising}. In our experiments, we use 4,000 images to compute the FID for each cluster. For cluster 1, we sample 1,000 images for each class from CIFAR-10 test data. For cluster 2, we sample 1,000 images of deer, dog, frog, and horse classes. The second one is the classification accuracy of eight classes: four from cluster 1 and four (deer, dog, frog, horse) from cluster 2. We train a simple CNN model (two convolutional layers, two fully connected layers, and one max pooling layer) on 8,000 samples (1,000 samples for each class) generated by each method. We then report the classification accuracy on the CIFAR-10 test dataset.

\paragraph{Results.} Table \ref{table:fid} summarizes the FID scores for different methods. We report the results for the model trained with cluster 1 (i.e., airplane, car, ship, truck) as the majority class. We can see that our method outperforms other baselines for the majority class. For the minority class, models trained without collaboration perform significantly worse compared to our method. Figure \ref{fig:1} illustrates the generated samples for different methods. Table \ref{table:test_acc} reports the test accuracy of different methods. Our method significantly outperforms non-collaborative DMs. We also demonstrate the sampling process for our method in Figure \ref{fig:3}. The global step 0 corresponds to the output of using the shared global denoiser. We can see that our method indeed provides strong privacy guarantees for each pixel.

\section{Discussion}\label{sec:disc}
In this work, we propose a novel framework that allows us to collaboratively train
diffusion models across decentralized, private datasets. We provide differential privacy guarantees for our method, and experimental results validate the effectiveness of our proposed method.  As for the future work, we plan to conduct more experiments on large-scale datasets and models to evaluate the performance of our methods thoroughly. We also plan to use privacy attack methods to evaluate our method's privacy guarantees systematically.

\bibliography{arxiv_ref}
\bibliographystyle{apalike}

\newpage
\appendix
\section{Additional Related Work}\label{sec:related}
\paragraph{Beyond DDPM.}
Since the introduction of DDPMs, the field of diffusion models has seen significant advancements aimed at improving computational efficiency, sample quality, and applicability to various domains. Building upon DDPMs, \textbf{Latent Diffusion Models (LDMs)}~\citep{rombach2022high} enhance computational efficiency by operating in a compressed latent space rather than directly in pixel space. Utilizing a pre-trained encoder-decoder architecture, LDMs reduce memory and computational requirements for high-resolution image synthesis while maintaining high-quality outputs. This approach has facilitated applications such as text-to-image generation, inpainting, and super-resolution. Notably, \textbf{Stable Diffusion}, based on the LDM framework, has popularized these models by providing an open-source and scalable implementation.

Another notable development is the \textbf{Denoising Diffusion Implicit Models (DDIMs)}~\citep{song2020denoising}, which introduces a non-Markovian diffusion process to accelerate sampling. DDIMs achieve faster generation times by reducing the number of required sampling steps without compromising output quality, addressing one of the primary limitations of earlier diffusion models.

Further advancements include the integration of classifier guidance and classifier-free guidance techniques~\citep{dhariwal2021diffusion}. These methods enhance controllability in the generation process, allowing for more precise adherence to desired attributes or conditions during sample generation.

Recent research has also explored the application of diffusion models beyond image generation. For instance, \textbf{Upsampling Diffusion Probabilistic Models (UDPMs)}~\citep{abu2023udpm} focus on upsampling tasks, generating high-resolution images from lower-resolution inputs. Additionally, diffusion models have been adapted for applications in audio generation, text synthesis, and even reinforcement learning scenarios~\cite{zhu2023diffusion}.

These developments reflect the rapid evolution of diffusion models, expanding their capabilities and applications across various fields.

\textbf{Diffusion Models as a Tool for FL Applications.} A separate line of work applies diffusion models to facilitate specific FL tasks rather than federated training of the diffusion model itself. \citet{sattarov2024fedtabdiff} developed FedTabDiff, which employs diffusion models to generate high-fidelity tabular data in FL settings without requiring centralized access to raw data. Other works explore how diffusion models can improve privacy-preserving FL through synthetic data augmentation or as generative priors in adversarial training. For example, FL frameworks incorporating diffusion models for one-shot learning and differential privacy constraints have been studied in \citep{lai2024demand}.

While these application-oriented studies demonstrate the versatility of diffusion models, they do not address the core challenges of federated training of generative models. The primary challenge remains how to train diffusion models in FL while mitigating privacy risks and ensuring meaningful personalization. Our work builds upon the first category by introducing a personalized federated diffusion model that overcomes these limitations.

\section{Proof of Main Theorem}
\subsection{Proof of Theorem \ref{thm:privacy}}\label{sec:app_proof}
We first introduce the definition of R\'enyi Differential Privacy (RDP) \citep{mironov2017renyi}. In our privacy analysis, we first use RDP to account for privacy loss and then translate the RDP guarantee to $(\epsilon,\delta)$-DP guarantee.
\begin{definition}[RDP] A randomized mechanism $\cA$ satisfies $(\gamma,\rho)$-R\'enyi differential privacy with $\gamma>1$ and $\rho>0$ if for adjacent datasets $D,D'\in\cD$ differing by one element,  $D_{\gamma}\big(\cA(D)||\cA(D')\big)=\log\EE_{\cA(D')}\big(\cA(D)/\cA(D')\big)^\gamma/(1-\gamma)\leq \rho$.
\end{definition}
Given a privacy guarantee in terms of RDP, we can transfer it to $(\epsilon,\delta)$-DP using the following lemma \citep{mironov2017renyi}.
\begin{lemma}[RDP to DP]\label{lemma:RDP_to_DP}
If a randomized mechanism $\cA$ satisfies $(\gamma,\rho)$-RDP, then $\cA$ satisfies $(\rho+\log(1/\delta)/(\gamma-1),\delta)$-DP for all $\delta\in(0,1)$.
\end{lemma}
To ensure the RDP, we need the following result for Gaussian mechanism \cite{mironov2017renyi}.

\begin{lemma}[Gaussian Mechanism]
\label{lemma:GaussianM_RDP}
Given a function $q$, the Gaussian Mechanism $\cA=q(D)+z$, where $z\sim N(0,\sigma^2\Ib)$, satisfies $(\gamma,\gamma S_2^2/(2\sigma^2))$-RDP, where $S_2$ is the $\ell_2$-sensitivity of $q$ and is defined as $S_2=\sup_{D,D'}\|q(D)-q(D')\|_2$ for two adjacent datasets $D,D'$ differing by one element.
\end{lemma}

Now, we are ready to provide the privacy guarantees of our method.
\begin{proof}[Proof of Theorem \ref{thm:privacy}]
    According to Algorithm \ref{alg:train_feddiff}, the training of the shared global denoiser $z_w$ is based on the noisy dataset $\tilde D=\{\tilde D_m\}_{m\in [M]}$. For each data point $\tilde x_0^{i,m}$ in $\tilde D$, it is generated by adding random Gaussian noise to the original data as follows: $\tilde{x}_0^{i,m}=\sqrt{\bar\alpha_{t_0}}x_0^{i,m}+\sqrt{1-\bar\alpha_{t_0}}z$ (see line 7 in Algorithm \ref{alg:train_feddiff}), where $z\sim \cN(0,\Ib)$. Therefore, we only need to prove the privacy guarantee for $\sqrt{\bar\alpha_{t_0}}x_0^{i,m}$ under Gaussian mechanism. By Lemma \ref{lemma:GaussianM_RDP}, we have that each data point in $\tilde D$ is $(\gamma,\gamma\tau)$-RDP with $\tau=2\bar \alpha_{t_0}C^2/(1-\bar \alpha_t)$. According to Lemma \ref{lemma:RDP_to_DP}, it is $(\epsilon,\delta)$-DP with $\epsilon=\gamma\tau+\log(1/\delta)/(\gamma-1)$. Therefore, we can choose $\gamma=1+\sqrt{\log(1/\delta)/\tau}$ to get the smallest $\epsilon=\tau+2\sqrt{\log(1/\delta)\tau}$. By plugging the value of $\tau$, it is $\Big(\frac{2\bar \alpha_{t_0}C^2}{1-\bar \alpha_t}+C\sqrt{\frac{8\bar \alpha_{t_0}\log(1/\delta)}{1-\bar \alpha_{t_0}}},\delta\Big)$-DP. Since the guarantee is for each data point, we prove the same level of LDP for the creation of $\tilde D$. As a result, by the post processing property of differential privacy, the shared global denoiser $z_w$ is $\Big(\frac{2\bar \alpha_{t_0}C^2}{1-\bar \alpha_t}+C\sqrt{\frac{8\bar \alpha_{t_0}\log(1/\delta)}{1-\bar \alpha_{t_0}}},\delta\Big)$-LDP. 
    
    Furthermore, we can also provide the privacy guarantees for each element in $x$ (e.g., each pixel in a given image). In this case, we only need to replace $C$ with $c$, where $c=\max_{i\in[d]}x[i]$ and $x[i]$ is the $i$-the coordinate of $x$. As a result, the shared global denoiser $z_w$ is $\Big(\frac{2\bar \alpha_{t_0}c^2}{(1-\bar \alpha_t)}+c\sqrt{\frac{8\bar \alpha_{t_0}\log(1/\delta)}{(1-\bar \alpha_{t_0})}},\delta\Big)$-DP for each element in $x$.
\end{proof}

\end{document}